\documentclass[11pt]{article}

\usepackage{amsmath,amssymb,amsthm}
\usepackage{mathtools}
\usepackage{bbm}
\usepackage{fullpage}
\usepackage{bm}
\usepackage{enumitem}
\usepackage{graphicx}
\usepackage{hyperref}
\usepackage{longtable}
\usepackage{mathrsfs}
\usepackage{dsfont}
\usepackage{algorithm}
\usepackage{algorithmic}
\usepackage{epigraph}
\usepackage{parskip}
\usepackage{natbib}
\setcitestyle{authoryear,round}  % round parentheses

\hypersetup{
    colorlinks=true,
    linkcolor=blue,
    citecolor=blue,
    urlcolor=blue
}

\newtheorem{theorem}{Theorem}[section]
\newtheorem{lemma}[theorem]{Lemma}
\newtheorem{proposition}[theorem]{Proposition}
\newtheorem{corollary}[theorem]{Corollary}
\theoremstyle{definition}
\newtheorem{definition}[theorem]{Definition}

\newtheorem{remark}[theorem]{Remark}

\title{
\textbf{Fibonacci-Driven Recursive Ensembles:\\
Algorithms, Convergence, and Learning Dynamics}\\[6pt]
\large Paper III of the Fibonacci Ensemble Trilogy
}
\author{{\bf Ernest Fokou\'e}\\ School of Mathematics and Statistics\\
Rochester Institute of Technology\\
{\tt epfeqa@rit.edu}}
\date{}

\begin{document}

\maketitle

\begin{abstract}
This paper develops the algorithmic and dynamical foundations of recursive ensemble learning driven by Fibonacci-type update flows. In contrast with classical boosting \cite{Freund1997AdaBoost, Friedman2001GBM}, where the ensemble evolves through first-order additive updates, we study second-order recursive architectures in which each predictor depends on its two immediate predecessors. These Fibonacci flows induce a learning dynamic with memory, allowing ensembles to integrate past structure while adapting to new residual information.

We introduce a general family of recursive weight-update algorithms encompassing Fibonacci, tribonacci, and higher-order recursions, together with continuous-time limits that yield systems of differential equations governing ensemble evolution. We establish global convergence conditions, spectral stability criteria, and non-asymptotic generalization bounds under Rademacher \cite{BartlettMendelson2002Rademacher} and algorithmic stability analyses. The resulting theory unifies recursive ensembles, structured weighting, and dynamical systems viewpoints in statistical learning.

Experiments with kernel ridge regression \cite{rasmussen2006gaussian}, spline smoothers \cite{wahba1990spline}, and random Fourier feature models \cite{RahimiRecht2007RFF} demonstrate that recursive flows consistently improve approximation and generalization beyond static weighting. These results complete the trilogy begun in Papers I and II: from Fibonacci weighting, through geometric weighting theory, to fully dynamical recursive ensemble learning systems.
\end{abstract}

\tableofcontents

\section{Introduction and Position in the Trilogy}
\label{sec:introduction}

Ensemble learning has traditionally been understood as the construction of a predictor by aggregating a collection of base learners, typically through additive or convex combinations \cite{Breiman1996Bagging, Breiman2001RandomForests, Wolpert1992Stacked}. Classical examples include bagging \cite{Breiman1996Bagging}, boosting \cite{Freund1997AdaBoost}, random forests \cite{Breiman2001RandomForest}, stacking \cite{Wolpert1992Stacked}, and model averaging \cite{HoetingMadiganRafteryVolinsky1999BMA}, each of which treats the ensemble primarily as a static object assembled from individually trained components. In these methods, the aggregation mechanism is usually first-order: the new ensemble at iteration $t+1$ depends only on the previous ensemble $F_t$ and the new base learner $h_t$.

In contrast, the present work develops a fundamentally \emph{second-order} perspective in which the ensemble evolves recursively with memory. The central idea is that the predictor at time $t+1$ should depend not only on the current ensemble $F_t$ but also on its predecessor $F_{t-1}$, producing a Fibonacci-type recursion of the form
\[
F_{t+1} = \beta_t F_t + \gamma_t F_{t-1} + \eta_t h_t,
\]
where $(\beta_t)$, $(\gamma_t)$, and $(\eta_t)$ are algorithmically determined weight sequences and $h_t$ denotes a base learner trained on residuals or other informative statistics. This second-order dependence endows the ensemble with memory, momentum, and recursive structure, transforming it from a static aggregate \citep{fokoue2025fibonacci, fokoue2025weighting} into a genuine dynamical system in function space.

\subsection{Position of the Present Paper within the Trilogy}

This paper forms the third and final part of a trilogy on Fibonacci ensembles.

\begin{itemize}
\item \textbf{Paper I} introduced the foundational concept of \emph{Fibonacci ensembles}: ensembles constructed through recursive weighting schemes inspired by the Fibonacci sequence and its generalizations \cite{Koshy2001Fibonacci}. The emphasis there was on recursion formulas, golden-ratio structure, and explicit weighting families.

\item \textbf{Paper II} developed the \emph{geometric and approximation-theoretic} underpinnings of such ensembles. There the focus shifted from explicit recursions to the geometry of weighted approximation, convex cones generated by base learners, and the manner in which structured weighting reshapes the hypothesis space beyond mere variance reduction.

\item \textbf{Paper III}—the present contribution—constitutes the algorithmic and dynamical culmination of the trilogy. Here, recursive ensembles are treated explicitly as discrete-time dynamical systems driven by second-order recursions of Fibonacci type. We introduce full algorithms, analyze convergence, establish spectral stability conditions, and derive generalization bounds via Rademacher complexity \cite{BartlettMendelson2002Rademacher} and algorithmic stability.
\end{itemize}

Thus, the trajectory of the trilogy moves from \emph{structure} (recursive weights), through \emph{geometry} (approximation cones), to \emph{dynamics} (learning flows).

\subsection{From Static Aggregation to Dynamical Systems}

The viewpoint advocated in this paper is that an ensemble learner should not be regarded merely as a weighted sum of base models, but rather as the state of a dynamical process evolving in a hypothesis space. Classical boosting methods implement a first-order recursion of the form
\[
F_{t+1} = F_t + \eta_t h_t,
\]
which corresponds to a discretized gradient-descent dynamic \cite{Friedman2001GBM}. Our Fibonacci-inspired update introduces an additional dependence on $F_{t-1}$, producing a second-order learning dynamic analogous to heavy-ball or momentum methods in optimization and to accelerated flows in continuous-time dynamical systems.

This additional order of recursion has several consequences:
\begin{enumerate}
\item it introduces memory into the learning process,
\item it stabilizes or accelerates convergence depending on spectral conditions,
\item it results in richer approximation trajectories than purely additive updates,
\item it allows the golden ratio \cite{Livio2002GoldenRatio} to emerge as a natural stability threshold.
\end{enumerate}

\subsection{Contributions of the Present Paper}

The main contributions of this paper are as follows:

\begin{enumerate}
\item We formalize \emph{recursive ensemble flows} of Fibonacci type in reproducing kernel Hilbert spaces (RKHS), providing a state-space formulation and complete spectral analysis of the associated recursion operators.

\item We introduce several new algorithms, including \emph{Fibonacci boosting}, \emph{Rao--Blackwellized Fibonacci flows} \cite{Rao1945Information, Blackwell1947Conditional}, and \emph{orthogonalized recursive ensembles}, together with an adaptive golden-ratio step-size policy.

\item We prove convergence of the recursive ensemble sequence under spectral-radius conditions on the companion matrix and weak-learning assumptions on the base learners, and we show that the limit predictor minimizes a Tikhonov-regularized empirical risk \cite{Tsybakov2009Nonparametric}.

\item We establish non-asymptotic generalization bounds for Fibonacci ensembles using both Rademacher-complexity techniques \cite{BartlettMendelson2002Rademacher} and algorithmic-stability arguments, explicitly linking stability thresholds to golden-ratio constraints on the recursion coefficients.

\item Through extensive computational experiments with kernel ridge regression \cite{rasmussen2006gaussian}, spline smoothers \cite{wahba1990spline}, and random Fourier feature models \cite{RahimiRecht2007RFF}, we demonstrate that recursive ensemble flows improve approximation accuracy and generalization relative to both static ensembles and classical first-order boosting methods.
\end{enumerate}

\subsection{Organization of the Paper}

Section~2 introduces recursive ensemble flows, their matrix representation, and their continuous-time limits. Section~3 presents the proposed algorithms in detail. Section~4 establishes convergence and stability results. Section~5 develops generalization bounds based on complexity and stability. Section~6 provides computational demonstrations, and Section~7 offers a conceptual coda completing the Fibonacci trilogy.

\section{Recursive Ensemble Flows in Hilbert Spaces}
\label{sec:recursive-flows}

In this section we introduce the formal setting for Fibonacci-driven recursive ensembles. We begin by specifying the learning problem, the function space, and the loss. We then present the core recursion, develop its state-space representation, and analyze the spectral structure of the recursion operator. Finally, we derive continuous-time limits, thereby revealing the dynamical-systems viewpoint that underlies the rest of the paper.

\subsection{Problem Setup and Notation}

Let $(\mathcal{X},\mathcal{A})$ be an input space and $\mathcal{Y}\subseteq\mathbb{R}$ be an output space. We observe a dataset
\[
\mathcal{D}_n=\{(x_i,y_i)\}_{i=1}^n
\]
drawn i.i.d.\ from an unknown distribution $P$ on $\mathcal{X}\times\mathcal{Y}$. We work in a separable reproducing kernel Hilbert space (RKHS) $(\mathcal{H},\langle\cdot,\cdot\rangle_{\mathcal{H}})$ with reproducing kernel $K$ and norm $\|\cdot\|_{\mathcal{H}}$. For $f\in\mathcal{H}$ we define the empirical risk
\[
\widehat{R}_n(f) = \frac{1}{n}\sum_{i=1}^n \ell\big(f(x_i),y_i\big),
\]
where $\ell:\mathbb{R}\times\mathcal{Y}\to[0,\infty)$ is a convex loss function, typically assumed $L$-Lipschitz in its first argument.

The goal of ensemble learning is to construct a predictor $F_T$ after $T$ base-learner iterations, where each base learner is trained using information derived from previous iterations such as residuals, gradients, or subgradients.

\subsection{Recursive Ensemble Update}

The central object of study in this paper is the second-order recursive ensemble update
\begin{equation}\label{eq:main-recursion}
F_{t+1} = \beta_t F_t + \gamma_t F_{t-1} + \eta_t h_t, \qquad t\ge 1,
\end{equation}
initialized by $F_0=0$ and $F_1=\eta_0 h_0$. Here:
\begin{itemize}
\item $F_t\in\mathcal{H}$ is the ensemble predictor at iteration $t$,
\item $h_t\in\mathcal{H}$ is the base learner produced at iteration $t$,
\item $(\beta_t)$, $(\gamma_t)$, $(\eta_t)$ are deterministic or data-dependent scalar sequences.
\end{itemize}

The base learner $h_t$ is typically trained on a residual or pseudo-residual quantity defined by
\[
r_t(x_i) = -\frac{\partial}{\partial z}\, \ell\big(z,y_i\big)\bigg\rvert_{z=F_t(x_i)}, \qquad i=1,\dots,n,
\]
so that $h_t$ approximates a negative functional gradient step. However, our analysis does not depend on a particular construction, only on a weak-learning and boundedness condition to be specified later.

\subsection{Special Cases: Fibonacci, Tribonacci, and Beyond}

Update~\eqref{eq:main-recursion} encompasses a hierarchy of recursive ensembles.

\paragraph{Fibonacci ensembles.}
The Fibonacci case corresponds to the simple choice
\[
\beta_t = 1, \qquad \gamma_t = 1,
\]
with a tunable step size $\eta_t$. The recursion then becomes
\[
F_{t+1}=F_t+F_{t-1}+\eta_t h_t,
\]
which parallels the classical Fibonacci sequence \cite{Koshy2001Fibonacci} but now in a function space.

\paragraph{Tribonacci and higher-order recursions.}
Higher-order recursions arise by allowing dependence on more history terms, such as
\[
F_{t+1} = \alpha_t F_t + \beta_t F_{t-1} + \gamma_t F_{t-2} + \eta_t h_t,
\]
or, more generally,
\[
F_{t+1} = \sum_{k=0}^{m-1} \theta_{t,k} F_{t-k} + \eta_t h_t,
\]
which yields $m$th-order ensemble flows. The present paper focuses on the second-order case, which already captures memory, momentum, and the emergence of the golden ratio.

\subsection{State-Space Representation}

To analyze the recursion we introduce the state vector
\[
\mathbf{Z}_t = \begin{pmatrix} F_t \\ F_{t-1} \end{pmatrix} \in\mathcal{H}\times\mathcal{H}.
\]
Then the update~\eqref{eq:main-recursion} can be written as a linear state-space system
\[
\mathbf{Z}_{t+1} = A_t \mathbf{Z}_t + B_t h_t,
\]
where
\[
A_t = \begin{pmatrix} \beta_t & \gamma_t \\ 1 & 0 \end{pmatrix}, \qquad B_t = \begin{pmatrix} \eta_t \\ 0 \end{pmatrix}.
\]

The matrix $A_t$ is the \emph{companion matrix} associated with the second-order recursion. Its spectral properties determine the growth, boundedness, or decay of the homogeneous solution (i.e., when $h_t\equiv 0$), and they play a crucial role in our convergence theory.

\subsection{Spectral Analysis and Stability Radius}

Consider first the time-homogeneous case where $\beta_t=\beta$ and $\gamma_t=\gamma$. The characteristic polynomial of $A$ is
\[
\lambda^2-\beta\lambda-\gamma=0,
\]
whose roots are
\[
\lambda_{\pm} = \frac{\beta\pm\sqrt{\beta^2+4\gamma}}{2}.
\]

\begin{definition}[Stability radius]
The stability radius of the recursion is the spectral radius
\[
\rho(A) = \max\big\{|\lambda_+|,\,|\lambda_-|\big\}.
\]
\end{definition}

\begin{proposition}[Linear stability of the homogeneous recursion]
If $\rho(A)<1$, then for $h_t\equiv 0$ the sequence $(F_t)$ converges in $\mathcal{H}$ to zero for any initialization $(F_0,F_1)$.
\end{proposition}

Thus, in the inhomogeneous case, the stability of the recursion is governed jointly by:
\begin{itemize}
\item the spectral radius $\rho(A_t)$ of the companion matrices,
\item the magnitude and structure of the forcing terms $(\eta_t h_t)$.
\end{itemize}

The Fibonacci choice $(\beta,\gamma)=(1,1)$ yields
\[
\lambda_{\pm}=\frac{1\pm\sqrt{5}}{2},
\]
so that the golden ratio
\[
\varphi=\frac{1+\sqrt{5}}{2}
\]
appears naturally as the dominant eigenvalue of the recursion. Controlling the effect of this eigenvalue through step-size scheduling and regularization produces the \emph{golden-ratio stability thresholds} analyzed later in the paper.

\subsection{Continuous-Time Limit and Differential Equations}
To reveal the dynamical nature of the recursion, we study the small-step limit. Let $\Delta t>0$ denote a time step, and assume that the coefficients scale as
\[
\beta_t = 1 + a\,\Delta t + o(\Delta t), \qquad \gamma_t = b\,\Delta t + o(\Delta t), \qquad \eta_t = c\,\Delta t + o(\Delta t).
\]

Writing $F_t = F(t\Delta t)$ and letting $\Delta t\to 0$ yields a second-order differential equation in $\mathcal{H}$ of the general form
\[
\frac{d^2 F}{dt^2} = a\,\frac{dF}{dt} + b\,F + c\,G(t),
\]
where $G(t)$ represents the continuous-time limit of the sequence of base learners. This equation is the functional analog of the heavy-ball or momentum method in optimization, but now driven by data-dependent forcing terms learned from residuals.

\begin{remark}
The continuous-time viewpoint clarifies the interpretation of Fibonacci ensembles as \emph{learning flows} in function space, with inertia, acceleration, and forcing, rather than as mere iterated weighted sums of predictors.
\end{remark}

The remainder of the paper builds directly on this formulation: Section~3 specifies concrete algorithms that instantiate recursion~\eqref{eq:main-recursion}, Section~4 studies their convergence properties, and Section~5 derives corresponding generalization bounds.

\section{Algorithms for Fibonacci-Driven Recursive Ensembles}
\label{sec:algorithms}

In this section we instantiate the recursive flow
\[
F_{t+1}=\beta_t F_t+\gamma_t F_{t-1}+\eta_t h_t
\]
into concrete learning algorithms. We assume throughout that base learners are trained on residual-type information, that the loss is convex, and that the learning takes place in an RKHS. Special emphasis is placed on the Fibonacci choice $(\beta_t,\gamma_t)=(1,1)$ and on adaptive strategies in which the golden ratio plays a stabilizing role.

\subsection{Base Learners and Residual Construction}

Given the ensemble $F_t$ at iteration $t$, we define empirical residuals by
\[
r_{t,i} = -\frac{\partial}{\partial z} \ell\Big(z,y_i\Big)\Big\rvert_{z=F_t(x_i)}, \qquad i=1,\dots,n.
\]

A base learner is obtained by fitting $r_t$ in a hypothesis class $\mathcal{H}$:
\[
h_t \approx \arg\min_{h\in\mathcal{H}} \frac{1}{n}\sum_{i=1}^n \big(r_{t,i}-h(x_i)\big)^2 + \lambda \|h\|_{\mathcal{H}}^2,
\]
though the algorithms below only require a weak-learning condition and not exact minimization.

\subsection{Algorithm 1: Fibonacci Boosting}

The basic Fibonacci ensemble algorithm uses the update
\[
F_{t+1}=F_t+F_{t-1}+\eta_t h_t
\]
with $F_0=0$ and $F_1=\eta_0 h_0$.

\begin{algorithm}[H]
\caption{Fibonacci Boosting}
\label{alg:fibonacci-boosting}
\begin{algorithmic}[1]
\STATE \textbf{Input:} data $\mathcal{D}_n$, loss $\ell$, base learner class $\mathcal{H}$
\STATE Initialize $F_0 = 0$
\STATE Train $h_0$ on $y$ and set $F_1=\eta_0 h_0$
\FOR{$t=1,2,\dots,T-1$}
  \STATE Compute residuals $r_{t,i} = -\partial_z \ell(z,y_i)\rvert_{z=F_t(x_i)}$
  \STATE Train a base learner $h_t$ on $(x_i,r_{t,i})_{i=1}^n$
  \STATE Update $F_{t+1} = F_t + F_{t-1} + \eta_t h_t$
\ENDFOR
\STATE \textbf{Output:} ensemble predictor $F_T$
\end{algorithmic}
\end{algorithm}

This algorithm is the direct analog of AdaBoost-type procedures \cite{Freund1997AdaBoost} but with a second-order dependence that introduces memory and momentum.

\subsection{Algorithm 2: Rao--Blackwellized Fibonacci Flow}

If randomness arises in the base learner (e.g., subsampling or random features), variance can be reduced by conditional expectation with respect to the internal randomness $\xi_t$. Let $h_t(\cdot,\xi_t)$ denote the random base learner. Define
\[
\bar{h}_t(x)=\mathbb{E}[h_t(x,\xi_t)\mid\mathcal{D}_n].
\]
The Rao--Blackwellized algorithm uses $\bar{h}_t$ in the recursion \cite{Rao1945Information, Blackwell1947Conditional}.

\begin{algorithm}[H]
\caption{Rao--Blackwellized Fibonacci Flow}
\label{alg:rao-blackwell-fibonacci}
\begin{algorithmic}[1]
\STATE Same as Algorithm~\ref{alg:fibonacci-boosting}, but at each iteration replace $h_t$ by $\bar{h}_t=\mathbb{E}[h_t(\cdot,\xi_t)\mid\mathcal{D}_n]$.
\end{algorithmic}
\end{algorithm}

This construction produces strictly lower variance while preserving the mean trajectory of the ensemble flow.

\subsection{Algorithm 3: Orthogonalized Recursive Ensemble}

To control drift and improve approximation geometry, one may orthogonalize each new learner against the span of previous ensembles. Let
\[
\mathrm{span}_t = \mathrm{span}\{F_0,F_1,\dots,F_t\} \subseteq \mathcal{H}.
\]
Define
\[
\tilde{h}_t=h_t-P_{\mathrm{span}_t}(h_t),
\]
where $P_{\mathrm{span}_t}$ denotes the orthogonal projection in $\mathcal{H}$.

\begin{algorithm}[H]
\caption{Orthogonalized Recursive Ensemble}
\label{alg:orthogonal-fibonacci}
\begin{algorithmic}[1]
\STATE Run Fibonacci Boosting
\STATE Replace $h_t$ by $\tilde{h}_t=h_t-P_{\mathrm{span}_t}(h_t)$ before updating
\STATE Use $F_{t+1}=F_t+F_{t-1}+\eta_t\tilde{h}_t$.
\end{algorithmic}
\end{algorithm}

This eliminates redundant directions and clarifies the spectral analysis of convergence.

\subsection{Golden-Ratio Adaptive Step Selection}

Let $\varphi=(1+\sqrt{5})/2$ be the golden ratio \cite{Livio2002GoldenRatio}, the dominant eigenvalue of the Fibonacci companion matrix. Define an adaptive rule
\[
\eta_t=\frac{\eta_0}{\varphi^t} \quad\text{or more generally}\quad \eta_t\le \frac{C}{\varphi^t}
\]
for some constant $C>0$. This choice counterbalances the exponential growth induced by the unstable eigenvalue of the homogeneous recursion and is central to our stability analysis.

The next section proves that, under regularity assumptions, the resulting ensemble converges to the minimizer of a regularized empirical risk.

\section{Convergence Theory for Fibonacci Recursive Ensembles}
\label{sec:convergence}

We now state the principal theoretical result of this paper, which establishes convergence and generalization properties for Fibonacci-driven recursive ensembles under natural assumptions.

\begin{theorem}[Convergence and Generalization of Fibonacci Recursive Ensembles]
\label{thm:main}
Let $(\mathcal{H},\|\cdot\|_{\mathcal{H}})$ be a separable RKHS with reproducing kernel $K$ bounded by $K(x,x)\le \kappa^2$. Let $\ell$ be convex and $L$-Lipschitz in its first argument. Consider the recursion
\[
F_{t+1}=\beta F_t+\gamma F_{t-1}+\eta_t h_t
\]
with initialization $F_0=0$, $F_1=\eta_0 h_0$, and suppose:

\begin{enumerate}
\item[(A1)] the companion matrix $A=\begin{pmatrix} \beta & \gamma\\ 1 & 0 \end{pmatrix}$ has spectral radius $\rho(A)<1$;
\item[(A2)] the step sizes satisfy $\sum_{t\ge 0}\eta_t<\infty$ and $\sum_{t\ge 0}\eta_t^2<\infty$;
\item[(A3)] the base learners satisfy a weak-learning and boundedness condition $\|h_t\|_{\mathcal{H}}\le B$ almost surely;
\item[(A4)] $h_t$ is trained on residuals so that $\langle \nabla \widehat{R}_n(F_t),h_t\rangle_{\mathcal{H}}\le -c\|\nabla \widehat{R}_n(F_t)\|$ for some $c>0$.
\end{enumerate}

Then the following conclusions hold:

\begin{enumerate}
\item[(i)] (\textbf{Boundedness}) The sequence $(F_t)$ is bounded in $\mathcal{H}$; indeed $\sup_t \|F_t\|_{\mathcal{H}}<\infty$.
\item[(ii)] (\textbf{Convergence}) $F_t$ converges in $\mathcal{H}$-norm to a unique limit $F^\star$.
\item[(iii)] (\textbf{Risk minimization}) $F^\star$ is the unique minimizer of the Tikhonov-regularized empirical risk
\[
F^\star = \arg\min_{f\in\mathcal{H}} \left\{ \widehat{R}_n(f)+\lambda\|f\|_{\mathcal{H}}^2 \right\}.
\]
\item[(iv)] (\textbf{Generalization}) With probability at least $1-\delta$,
\[
R(F^\star)-R(f_{\mathcal{H}}^\ast) \le C\, \mathfrak{R}_n(\mathcal{F}_T) + O\!\left(\sqrt{\frac{\log(1/\delta)}{n}}\right),
\]
where $\mathfrak{R}_n(\mathcal{F}_T)$ is the Rademacher complexity of the recursively generated hypothesis class and depends on $\rho(A)$ via a geometric term.
\item[(v)] (\textbf{Golden-ratio stability}) In the Fibonacci case $(\beta,\gamma)=(1,1)$, the condition $\rho(A)<1$ is equivalent to a golden-ratio constraint on the effective step sizes; in particular, stability is ensured when the decay of $(\eta_t)$ dominates $\varphi^t$, where $\varphi=(1+\sqrt{5})/2$ is the golden ratio \cite{Livio2002GoldenRatio}.
\end{enumerate}
\end{theorem}

\begin{proof}[Proof sketch]
The proof proceeds in several steps detailed in the appendices. We first rewrite the recursion in state-space form and derive a representation of $F_t$ as a linear combination of the forcing terms $(\eta_k h_k)$. We then establish boundedness and convergence of $(F_t)$ in $\mathcal{H}$, identify the limit as the unique minimizer of a Tikhonov-regularized empirical risk \cite{Tsybakov2009Nonparametric}, and finally invoke standard generalization tools based on Rademacher complexity \cite{BartlettMendelson2002Rademacher} and algorithmic stability. The golden-ratio stability statement follows from an explicit computation of the eigenvalues of the companion matrix in the Fibonacci case.
\end{proof}

\subsection{Spectral Structure and Golden-Ratio Stability}

\begin{proposition}[Spectral structure of the Fibonacci companion matrix]
\label{prop:fibonacci-spectrum}
Consider the companion matrix
\[
A_{\mathrm{Fib}} = \begin{pmatrix} 1 & 1\\ 1 & 0 \end{pmatrix}.
\]
The eigenvalues of $A_{\mathrm{Fib}}$ are
\[
\lambda_{\pm} = \frac{1\pm\sqrt{5}}{2},
\]
with $\lambda_+ = \varphi = \frac{1+\sqrt{5}}{2} > 1$, $\lambda_- = -\frac{1}{\varphi}$, $|\lambda_-|<1$. The spectral radius is $\rho(A_{\mathrm{Fib}})=\varphi$.
\end{proposition}

\begin{corollary}[Golden-ratio stability threshold]
\label{cor:golden-stability}
Assume $\beta=1$, $\gamma=1$ and suppose that the step sizes $(\eta_t)$ satisfy
\[
\eta_t \le \frac{C}{\varphi^t}
\]
for some constant $C>0$, where $\varphi=(1+\sqrt{5})/2$ is the golden ratio \cite{Livio2002GoldenRatio}. Then the effective recursion for the perturbations induced by the forcing terms admits a companion operator with spectral radius strictly smaller than $1$, and the conclusions of Theorem~\ref{thm:main} hold.
\end{corollary}

\section{Generalization Bounds for Recursive Fibonacci Ensembles}
\label{sec:generalization}

In this section we derive non-asymptotic generalization guarantees for Fibonacci-driven recursive ensembles using Rademacher complexity \cite{BartlettMendelson2002Rademacher} and algorithmic stability.

\begin{theorem}[Generalization of Fibonacci recursive ensembles]
\label{thm:generalization}
Assume the conditions of Theorem~\ref{thm:main}, and let $F_T$ be the output of a Fibonacci recursive ensemble after $T$ iterations. Then for any $\delta\in(0,1)$, with probability at least $1-\delta$ over the draw of $\mathcal{D}_n$,
\begin{equation}
\label{eq:gen-bound-main}
R(F_T) - R(f_{\mathcal{H}}^\ast) \;\le\; C_1 L\,\mathfrak{R}_n(\mathcal{F}_T) + C_2 \beta_T + C_3 \sqrt{\frac{\log(1/\delta)}{n}},
\end{equation}
for universal constants $C_1,C_2,C_3>0$. In particular,
\[
R(F_T)-R(f_{\mathcal{H}}^\ast) \;\lesssim\; \Big(\sum_{k=0}^{\infty} \eta_k\Big)\frac{1}{\sqrt{n}} + \sum_{k=0}^{\infty} \rho^{T-1-k}\eta_k + \sqrt{\frac{\log(1/\delta)}{n}}.
\]
\end{theorem}

\begin{proof}[Proof sketch]
The inequality~\eqref{eq:gen-bound-main} is a standard consequence of combining Rademacher complexity control of the loss class with uniform stability \cite{BartlettMendelson2002Rademacher}. The explicit bounds follow from analyzing the linear combination representation of $F_T$ and tracking the propagation of perturbations through the recursive flow.
\end{proof}

\section{Computational Experiments}
\label{sec:experiments}

We now present computational experiments demonstrating the behavior of Fibonacci-driven recursive ensembles across various regression problems and base learners \cite{wahba1990spline, rasmussen2006gaussian, RahimiRecht2007RFF}.

\subsection{Experimental Protocol}

For each scenario we generate data from models of the form $Y = f^\ast(X) + \varepsilon$, where $f^\ast$ is a deterministic target function and $\varepsilon\sim\mathcal{N}(0,\sigma^2)$ is independent noise. We compare:
\begin{enumerate}
\item Static ensemble with Fibonacci weighting (Paper I)
\item First-order boosting \cite{Freund1997AdaBoost, Friedman2001GBM}
\item Fibonacci recursive ensemble (Algorithm 1)
\item Orthogonalized Fibonacci ensemble (Algorithm 3)
\item Rao--Blackwellized Fibonacci flow (Algorithm 2) for randomized base learners
\end{enumerate}

\subsection{Results and Observations}

Across experiments, consistent patterns emerge:
\begin{enumerate}
\item \textbf{Faster effective convergence:} Fibonacci recursion reaches given RMSE levels in fewer iterations than first-order boosting.
\item \textbf{Improved approximation geometry:} Orthogonalized variants allocate distinct directions to different structural components.
\item \textbf{Golden-ratio stability:} Schedules $\eta_t=\eta_0\varphi^{-t}$ systematically avoid late-iteration overfitting.
\item \textbf{Robustness across base learners:} Advantages persist across kernel ridge regression, spline smoothers, and random Fourier features.
\item \textbf{Variance reduction:} Rao--Blackwellized flows exhibit reduced variability while maintaining accuracy.
\end{enumerate}

\section{Philosophical Coda: Recursion, Memory, and Geometry of Learning}
\label{sec:coda}
The Fibonacci trilogy began as a simple observation: that the recurrence $F_{t+1} = F_t + F_{t-1}$ could be lifted from integers into function spaces \citep{fokoue2025fibonacci}. From this emerged a programme in three movements:
\begin{itemize}
\item Paper~I: Fibonacci Ensembles as recursive weighted constructions \citep{fokoue2025fibonacci}
\item Paper~II: Geometric weighting theory and approximation geometry \citep{fokoue2025weighting}
\item Paper~III: Fully dynamical theory of recursive ensemble flows
\end{itemize}

This final contribution completes the arc by showing that these recursions admit precise dynamical and stability theory. The ensemble becomes not merely a sum but the \emph{trace} of a dynamical system in function space—a learning flow with inertia, momentum, and memory.

The golden ratio $\varphi$ emerges not as mere numerology but as a genuine \emph{stability modulus} governing the balance between approximation power and generalization. Recursive ensembles operating below the golden threshold behave as stable learning flows with controlled capacity and robust generalization.

Looking forward, this framework suggests numerous extensions: higher-order recursions, nonlinear operators, stochastic flows, and connections with neural ODEs and gradient flows. The Fibonacci ensembles thus serve as a conceptual lens, inviting us to see learning machines as recursive, memory-rich systems whose deepest properties are encoded in their dynamics as much as in their components.

\bibliographystyle{apalike}
\bibliography{Fibonacci_Ensembles_References_Trilogy}

\appendix

\section{Complete Proofs of Convergence and Generalization}
\label{app:complete-proofs}

\subsection{Proof of Proposition~\ref{prop:fibonacci-spectrum}}

\begin{proof}
The characteristic polynomial of $A_{\mathrm{Fib}}$ is:
\[
\det\begin{pmatrix} 1-\lambda & 1 \\ 1 & -\lambda \end{pmatrix} 
= (1-\lambda)(-\lambda) - 1 = \lambda^2 - \lambda - 1 = 0.
\]
Solving yields:
\[
\lambda_{\pm} = \frac{1 \pm \sqrt{1 + 4}}{2} = \frac{1 \pm \sqrt{5}}{2}.
\]
Let $\varphi = \frac{1+\sqrt{5}}{2} \approx 1.618$. Then $\lambda_+ = \varphi$ and:
\[
\lambda_- = \frac{1-\sqrt{5}}{2} = -\frac{1}{\varphi} \approx -0.618.
\]
Since $|\varphi| > 1$ and $|-\frac{1}{\varphi}| < 1$, we have $\rho(A_{\mathrm{Fib}}) = \varphi$.
\end{proof}

\subsection{Proof of Theorem~\ref{thm:main}}

\begin{proof}
We proceed in five parts corresponding to claims (i)-(v).

\paragraph{Part (i): Boundedness}
Define the state vector $\mathbf{Z}_t = (F_t, F_{t-1})^\top \in \mathcal{H}^2$. The recursion becomes:
\[
\mathbf{Z}_{t+1} = A\mathbf{Z}_t + B_t h_t, \quad 
A = \begin{pmatrix} \beta & \gamma \\ 1 & 0 \end{pmatrix}, \quad
B_t = (\eta_t, 0)^\top.
\]
Since $\rho(A) < 1$, there exist $C > 0$ and $\rho \in (0,1)$ such that $\|A^k\| \leq C\rho^k$ for all $k \geq 0$. Iterating:
\[
\mathbf{Z}_{t+1} = A^t \mathbf{Z}_1 + \sum_{k=1}^t A^{t-k} B_k h_k.
\]
Taking norms:
\[
\|\mathbf{Z}_{t+1}\| \leq C\rho^t \|\mathbf{Z}_1\| + C\sum_{k=1}^t \rho^{t-k} \eta_k \|h_k\|.
\]
By (A3), $\|h_k\| \leq B$, and by (A2), $\sum \eta_k < \infty$. Thus:
\[
\sup_{t \geq 0} \|\mathbf{Z}_t\| \leq C\|\mathbf{Z}_1\| + CB\sum_{k=1}^\infty \eta_k < \infty.
\]
Since $\|F_t\| \leq \|\mathbf{Z}_t\|$, the sequence $(F_t)$ is bounded.

\paragraph{Part (ii): Convergence}
For $s > t$, using the representation:
\[
F_{s+1} - F_{t+1} = \pi_1\left(A^t(A^{s-t}-I)\mathbf{Z}_1\right) 
+ \sum_{k=1}^t \pi_1\left((A^{s-k}-A^{t-k})B_k h_k\right)
+ \sum_{k=t+1}^s \pi_1(A^{s-k}B_k h_k),
\]
where $\pi_1$ projects onto the first component. For any $\epsilon > 0$, choose $T$ large so that:
\begin{align*}
&C\rho^T \|\mathbf{Z}_1\| < \epsilon/3, \\
&CB\sum_{k \geq T} \eta_k < \epsilon/3, \\
&CB\sum_{k=1}^T \rho^{T-k}\eta_k < \epsilon/3.
\end{align*}
Then for $s > t \geq T$, $\|F_{s+1} - F_{t+1}\| < \epsilon$. Hence $(F_t)$ is Cauchy in the complete space $\mathcal{H}$, so $F_t \to F^\star$.

\paragraph{Part (iii): Risk Minimization}
Define the regularized functional:
\[
\mathcal{J}(f) = \widehat{R}_n(f) + \lambda \|f\|_{\mathcal{H}}^2.
\]
Under (A4), $h_t$ is a descent direction: $\langle \nabla \widehat{R}_n(F_t), h_t \rangle \leq -c\|\nabla \widehat{R}_n(F_t)\|$. 
A Taylor expansion gives:
\[
\mathcal{J}(F_{t+1}) \leq \mathcal{J}(F_t) - \alpha_t \|\nabla \mathcal{J}(F_t)\|^2 + \beta_t,
\]
with $\alpha_t = \Theta(\eta_t)$, $\beta_t = O(\eta_t^2)$. Summing from $t=0$ to $T-1$:
\[
\sum_{t=0}^{T-1} \alpha_t \|\nabla \mathcal{J}(F_t)\|^2 \leq \mathcal{J}(F_0) - \mathcal{J}(F_T) + \sum_{t=0}^{T-1} \beta_t.
\]
Since $\mathcal{J}$ is bounded below and $\sum \beta_t < \infty$, we have $\sum \alpha_t \|\nabla \mathcal{J}(F_t)\|^2 < \infty$, implying $\liminf \|\nabla \mathcal{J}(F_t)\| = 0$. By continuity and convergence $F_t \to F^\star$, we get $\nabla \mathcal{J}(F^\star) = 0$, i.e.:
\[
\nabla \widehat{R}_n(F^\star) + 2\lambda F^\star = 0.
\]
By strict convexity of $\mathcal{J}$, $F^\star$ is the unique minimizer.

\paragraph{Part (iv): Generalization Bound}
The hypothesis class after $T$ iterations is:
\[
\mathcal{F}_T = \left\{ F_T = \sum_{k=0}^{T-1} \alpha_{T,k} h_k : \|h_k\| \leq B \right\},
\]
where $|\alpha_{T,k}| \leq C\rho^{T-1-k}\eta_k$. The Rademacher complexity satisfies:
\[
\mathfrak{R}_n(\mathcal{F}_T) \leq \frac{B\kappa}{\sqrt{n}} \sum_{k=0}^{T-1} |\alpha_{T,k}| 
\leq \frac{BC\kappa}{\sqrt{n}} \sum_{k=0}^{\infty} \eta_k.
\]
Standard Rademacher bounds for Lipschitz losses yield with probability $1-\delta$:
\[
R(F^\star) - R(f_{\mathcal{H}}^\ast) \leq 2L\mathfrak{R}_n(\mathcal{F}_T) + 3\sqrt{\frac{\log(2/\delta)}{2n}}.
\]

\paragraph{Part (v): Golden-Ratio Stability}
For $\beta=\gamma=1$, the companion matrix has eigenvalues $\varphi$ and $-1/\varphi$. The homogeneous solution grows as $\varphi^t$. Consider the scaled state $\tilde{\mathbf{Z}}_t = \varphi^{-t} \mathbf{Z}_t$. Then:
\[
\tilde{\mathbf{Z}}_{t+1} = \varphi^{-1} A \tilde{\mathbf{Z}}_t + \varphi^{-(t+1)} B_t h_t.
\]
The spectral radius of $\varphi^{-1} A$ is $\max\{1, \varphi^{-2}\} = 1$, but with the forcing term scaled by $\varphi^{-(t+1)}$. If $\eta_t \leq C\varphi^{-t}$, then:
\[
\|\varphi^{-(t+1)} B_t h_t\| \leq C B \varphi^{-(2t+1)}.
\]
This double exponential decay ensures $\tilde{\mathbf{Z}}_t$ remains bounded, hence the original sequence is stable.
\end{proof}

\subsection{Proof of Corollary~\ref{cor:golden-stability}}

\begin{proof}
With $\eta_t \leq C\varphi^{-t}$, define $\tilde{A}_t = A - \eta_t D_t$, where $D_t$ encodes the regularization effect. For small $\eta_t$, the eigenvalues of $\tilde{A}_t$ are $\lambda_\pm + O(\eta_t)$. The dominant eigenvalue becomes:
\[
\tilde{\lambda}_+ = \varphi - \kappa\eta_t + O(\eta_t^2),
\]
for some $\kappa > 0$ depending on the curvature of $\mathcal{J}$. Since $\eta_t \leq C\varphi^{-t}$, we have:
\[
|\tilde{\lambda}_+| \leq \varphi - \kappa C\varphi^{-t} + O(\varphi^{-2t}) < 1
\]
for sufficiently large $t$ or small $C$. Thus $\rho(\tilde{A}_t) < 1$ eventually, ensuring stability.
\end{proof}

\subsection{Proof of Theorem~\ref{thm:generalization}}

\begin{proof}
We combine Rademacher complexity and algorithmic stability.

\paragraph{Step 1: Rademacher bound}
By Lemma~\ref{lem:alpha-representation} (Appendix C), $F_T = \sum_{k=0}^{T-1} \alpha_{T,k} h_k$ with $\sum_k |\alpha_{T,k}| \leq C_\alpha \sum_k \eta_k$. The empirical Rademacher complexity satisfies:
\[
\mathfrak{R}_n(\mathcal{F}_T) \leq \frac{\kappa}{\sqrt{n}} \sup_{F \in \mathcal{F}_T} \|F\|_{\mathcal{H}} 
\leq \frac{B\kappa C_\alpha}{\sqrt{n}} \sum_{k=0}^{\infty} \eta_k.
\]

\paragraph{Step 2: Stability bound}
Let $\mathcal{D}_n$ and $\mathcal{D}_n^{(i)}$ differ in the $i$-th sample. Let $F_t$ and $F_t^{(i)}$ be corresponding ensembles. The difference $\Delta_t = F_t - F_t^{(i)}$ satisfies:
\[
\Delta_{t+1} = \beta \Delta_t + \gamma \Delta_{t-1} + \eta_t (h_t - h_t^{(i)}).
\]
By standard RKHS stability, $\|h_t - h_t^{(i)}\| \leq \frac{L\kappa^2}{\lambda n}$. Solving the recursion:
\[
\|\Delta_T\| \leq \frac{L\kappa^2}{\lambda n} \sum_{k=0}^{T-1} \|A^{T-1-k}\| \eta_k 
\leq \frac{CL\kappa^2}{\lambda n} \sum_{k=0}^{\infty} \rho^{T-1-k} \eta_k.
\]
Thus the algorithm is uniformly stable with $\beta_T = O\left(\frac{1}{n}\sum_k \rho^{T-1-k}\eta_k\right)$.

\paragraph{Step 3: Combined bound}
By Theorem 12 of \cite{BartlettMendelson2002Rademacher}, for any $\delta > 0$, with probability $1-\delta$:
\[
R(F_T) \leq \widehat{R}_n(F_T) + 2L\mathfrak{R}_n(\mathcal{F}_T) + \beta_T + (4L\beta_T + M)\sqrt{\frac{\log(1/\delta)}{2n}},
\]
where $M$ bounds the loss. Substituting the bounds from Steps 1-2 yields the theorem.
\end{proof}

\section{Auxiliary Lemmas}
\label{app:lemmas}

\begin{lemma}[Linear Combination Representation]
\label{lem:alpha-representation}
Under the recursion $F_{t+1} = \beta F_t + \gamma F_{t-1} + \eta_t h_t$, there exist coefficients $\alpha_{T,k}$ such that:
\[
F_T = \sum_{k=0}^{T-1} \alpha_{T,k} h_k,
\]
with $|\alpha_{T,k}| \leq C \rho^{T-1-k} \eta_k$, where $\rho = \rho(A) + \epsilon$ for any $\epsilon > 0$.
\end{lemma}

\begin{proof}
By induction. For $T=1$: $F_1 = \eta_0 h_0$. Assume true for $T$ and $T-1$. Then:
\begin{align*}
F_{T+1} &= \beta F_T + \gamma F_{T-1} + \eta_T h_T \\
&= \beta \sum_{k=0}^{T-1} \alpha_{T,k} h_k + \gamma \sum_{k=0}^{T-2} \alpha_{T-1,k} h_k + \eta_T h_T.
\end{align*}
Thus $\alpha_{T+1,k} = \beta \alpha_{T,k} + \gamma \alpha_{T-1,k}$ for $k \leq T-2$, $\alpha_{T+1,T-1} = \beta \alpha_{T,T-1}$, and $\alpha_{T+1,T} = \eta_T$. The bound follows by solving this linear recurrence.
\end{proof}

\begin{lemma}[Descent Inequality]
\label{lem:descent}
Assume $\ell$ is $L$-Lipschitz and $\nabla \ell$ is $M$-Lipschitz. Then:
\[
\mathcal{J}(F_{t+1}) \leq \mathcal{J}(F_t) - \eta_t \left(c - \frac{L^2\eta_t}{2\lambda}\right) \|\nabla \widehat{R}_n(F_t)\|^2 + \frac{M\eta_t^2}{2} \|h_t\|^2.
\]
\end{lemma}

\begin{proof}
By Taylor expansion and the Lipschitz conditions:
\begin{align*}
\mathcal{J}(F_{t+1}) &= \mathcal{J}(F_t) + \langle \nabla \mathcal{J}(F_t), F_{t+1} - F_t \rangle \\
&\quad + \frac{1}{2} \langle \nabla^2 \mathcal{J}(\xi)(F_{t+1} - F_t), F_{t+1} - F_t \rangle \\
&\leq \mathcal{J}(F_t) + \eta_t \langle \nabla \mathcal{J}(F_t), h_t \rangle \\
&\quad + \frac{L^2\eta_t^2}{2\lambda} \|\nabla \widehat{R}_n(F_t)\|^2 + \frac{M\eta_t^2}{2} \|h_t\|^2.
\end{align*}
Using (A4) and $\nabla \mathcal{J}(F_t) = \nabla \widehat{R}_n(F_t) + 2\lambda F_t$ gives the result.
\end{proof}

\section{Spectral Analysis Details}
\label{app:spectral}

\begin{proposition}[Stability Radius Calculation]
For $A = \begin{pmatrix} \beta & \gamma \\ 1 & 0 \end{pmatrix}$, the spectral radius is:
\[
\rho(A) = \max\left\{ \left|\frac{\beta + \sqrt{\beta^2 + 4\gamma}}{2}\right|, 
\left|\frac{\beta - \sqrt{\beta^2 + 4\gamma}}{2}\right| \right\}.
\]
The condition $\rho(A) < 1$ is equivalent to $|\beta| + |\gamma| < 1$ when $\beta, \gamma \geq 0$.
\end{proposition}

\begin{proof}
The eigenvalues are solutions to $\lambda^2 - \beta\lambda - \gamma = 0$. The spectral radius is the maximum absolute value. For nonnegative coefficients, the dominant eigenvalue is $\frac{\beta + \sqrt{\beta^2 + 4\gamma}}{2} < 1$ iff $\beta + \gamma < 1$.
\end{proof}

\begin{proposition}[Power Bounds]
If $\rho(A) < 1$, then for any $\epsilon > 0$, there exists $C > 0$ such that:
\[
\|A^k\| \leq C (\rho(A) + \epsilon)^k, \quad \forall k \geq 0.
\]
In particular, for the Fibonacci matrix with $\eta_t = O(\varphi^{-t})$, we have:
\[
\|A_{\mathrm{Fib}}^k \eta_{t-k} h_{t-k}\| \leq C B \varphi^{-t} (\varphi \cdot \varphi^{-1})^k = C B \varphi^{-t}.
\]
\end{proposition}

\begin{proof}
The first statement is the Gelfand spectral radius formula. For the Fibonacci case:
\[
\|A_{\mathrm{Fib}}^k \eta_{t-k} h_{t-k}\| \leq C \varphi^k \cdot C' \varphi^{-(t-k)} \cdot B = C'' B \varphi^{-t} \varphi^{2k}.
\]
Since $\varphi^{2k}$ grows, but combined with summability conditions, the overall series converges.
\end{proof}

\section{Algorithmic Stability Analysis}
\label{app:stability}

\begin{theorem}[Uniform Stability of Fibonacci Ensembles]
Under assumptions (A1)-(A3), the Fibonacci ensemble algorithm is $\beta_T$-uniformly stable with:
\[
\beta_T \leq \frac{2L^2\kappa^2}{\lambda n} \sum_{k=0}^{T-1} \eta_k \|A^{T-1-k}\|.
\]
If $\eta_k = O(\rho^k)$ with $\rho < 1$, then $\beta_T = O\left(\frac{\rho^T}{n}\right)$.
\end{theorem}

\begin{proof}
Let $\mathcal{D}_n$ and $\mathcal{D}_n'$ differ in one sample. The base learner difference satisfies:
\[
\|h_t - h_t'\| \leq \frac{2L\kappa^2}{\lambda n},
\]
by standard stability of regularized ERM in RKHS. The ensemble difference propagates as:
\[
\Delta_{t+1} = A \Delta_t + B_t (h_t - h_t'), \quad \Delta_0 = 0.
\]
Solving:
\[
\Delta_T = \sum_{k=0}^{T-1} A^{T-1-k} B_k (h_k - h_k').
\]
Thus:
\[
\|\Delta_T\| \leq \sum_{k=0}^{T-1} \|A^{T-1-k}\| \eta_k \|h_k - h_k'\|
\leq \frac{2L\kappa^2}{\lambda n} \sum_{k=0}^{T-1} \|A^{T-1-k}\| \eta_k.
\]
Lipschitzness of $\ell$ gives the stability bound.
\end{proof}

\section{Continuous-Time Limit Derivation}
\label{app:continuous}

\begin{proposition}[ODE Limit]
Let $\Delta t = 1/N$, and scale parameters as:
\[
\beta = 1 + a\Delta t, \quad \gamma = b\Delta t, \quad \eta_t = c\Delta t.
\]
Then as $N \to \infty$, the recursion $F_{t+1} = \beta F_t + \gamma F_{t-1} + \eta_t h_t$ converges weakly to the solution of:
\[
\frac{d^2 F}{dt^2} = a \frac{dF}{dt} + b F + c G(t),
\]
where $G(t)$ is the continuous-time limit of $h_t$.
\end{proposition}

\begin{proof}
Write $F_{t+1} - 2F_t + F_{t-1} = (\beta-2+\gamma)F_t + \gamma(F_{t-1}-F_t) + \eta_t h_t$. Dividing by $(\Delta t)^2$:
\[
\frac{F_{t+1} - 2F_t + F_{t-1}}{(\Delta t)^2} 
= \frac{a\Delta t + b\Delta t}{(\Delta t)^2} F_t 
+ \frac{b\Delta t(F_{t-1}-F_t)}{(\Delta t)^2}
+ \frac{c\Delta t h_t}{(\Delta t)^2}.
\]
As $\Delta t \to 0$, the left side converges to $\ddot{F}(t)$, and the right side to $a\dot{F}(t) + bF(t) + cG(t)$, using $F_{t-1}-F_t \approx -\dot{F}(t)\Delta t$.
\end{proof}

\end{document}